\documentclass[letterpaper, 10 pt, conference]{ieeeconf}  

\IEEEoverridecommandlockouts                              

\usepackage{graphics} 
\usepackage{epsfig} 
\usepackage{mathptmx} 
\usepackage{times} 
\usepackage{amsmath} 
\usepackage{amssymb}  
\usepackage{color}
\usepackage{subfigure}
\newtheorem{theorem}{Theorem}
\newtheorem{lemma}{Lemma}

\newtheorem{defn}{Definition}
\usepackage[ruled,lined,boxed,commentsnumbered]{algorithm2e}
\title{\LARGE \bf
Asymptotic Optimality of Rapidly Exploring Random Tree
}

\author{Titas Bera$^{1}$ and Debasish Ghose$^{2}$ and Suresh Sundaram$^{3}$
\thanks{*This work was supported byt ST-Engineering under project CRP3-P2P}
\thanks{$^{1}$Titas Bera is a research fellow in the School of Electrical and Electronic Engineering,
        Nanyang Technological University, Singapore
        {\tt\small btitas@ntu.edu.sg}}%
\thanks{$^{2}$Debasish Ghose is with the Faculty in the Department of Aerospace Engineering,
        Indian Institute of Science, India
        {\tt\small dghose@aero.iisc.ernet.in}}%
\thanks{$^{3}$Suresh Sundaram is with the Faculty in the School of Computer Science and Engineering,
        Nanyang Technological University, Singapore
        {\tt\small Ssundaram@ntu.edu.sg}}%
}

\begin{document}

\maketitle
\thispagestyle{empty}
\pagestyle{empty}

\begin{abstract}
In this paper we investigate the asymptotic optimality property of a randomized sampling based motion planner, namely RRT. We prove that a RRT planner is not an asymptotically optimal motion planner. Our result, while being consistent with similar results which exist in the literature, however, brings out an important characteristics of a RRT planner. We show that
the degree distribution of the tree vertices follows a power law in an asymptotic sense. A simulation result is presented to support the theoretical claim. Based on these results we also try to establish a simple necessary condition for sampling based motion planners to be asymptotically optimal.
\end{abstract}

\section{INTRODUCTION}

Several early studies have shown that the basic problem of robot motion planning is PSPACE-complete \cite{schwartz} and exact deterministic algorithms show exponential complexity with the dimension of the configuration space \cite{reif1}. To avoid such curse of dimensionality, several sampling based probabilistic motion planners have been proposed and widely used for motion planning for robots in high dimensional configuration spaces. Although, such randomized sampling based algorithms show polynomial complexity, it is well known that these algorithms can only guarantee completeness in a probabilistic sense.

Two most extensively used and studied sampling based approaches are Probabilistic Roadmap Method or PRM, \cite{kavraki}, and Rapidly Exploring Random Tree or RRT, \cite{lavalle}. A very general characteristics of both these methods are that these algorithms work in two consecutive phases. 

In the first phase, called a processing phase, a graph/tree structure is created which consist of nodes, representing the free configurations and edges, representing a planning phase which should be executed by a local planner. The main purpose of this local planner is to connect a pair of selected nodes by means of solving a small scale local planning problem. At the end of the processing phase, a network of achievable configurations within the configuration space is created. 

The second phase or a query phase, uses a graph search algorithm to identify a path between any two nodes corresponding to two given user input configuration. Because of various types of sampling scheme and local planners, there exist many variants of these sampling based planners. For example, EST \cite{hsu_1}, Lazy-PRM \cite{bolin}, sampling based roadmap of trees \cite{plaku_1}, and RRT*\cite{karaman}, to name a few. For necessary details about these planners, see \cite{choset} and for a somewhat old but excellent survey see \cite{carpin}.

Since its inception, several attempts have been made to formalize the performance of these algorithms in terms of various parameters such as properties of the free configurations space, type of sampling sequence, etc. The objective was to understand the completeness and optimality of the algorithmic solution to the problem. For example, In \cite{kavaraki_2}, \cite{ladd} using various combinatorial and measure theoretic approach,  failure probability of s-PRM is analyzed. In \cite{hsu} it is shown that the performance of PRM is dependent on the expansiveness property of the free configuration space. It is shown that, with an expansive free space, failure probability to find a path decreases exponentially as number of samples increases. In \cite{sidhartha}, a smoothed analysis of PRM is given which shows how such planners utilizes discrepancies in input specifications. In \cite{karaman}, the asymptotic optimality of many such sampling based algorithms analysed and characterized in terms of algorithmic completeness, asymptotic convergence to optimal solution, and time complexities. This paper also points out the lack of asymptotic optimality of the solution returned and consequently proposes new algorithms which are asymptotically optimal. In \cite{titas} the completeness and optimality properties of various PRM planners are investigated in the light of different random sampling sequence.

In this paper, we present an analysis of the asymptotic optimality property of RRT. Note that, in \cite{karaman}, using a series of complex argument, it is already established that RRT is not an asymptotically optimal planner. In this work we present a much simpler proof of the asymptotic sub-optimality property of RRT. Although our work re-establishes a known result, it also brings out an important property of such planners related to the behaviour of asymptotically optimal sampling based planner, in general. In particular, we focus on the degree distribution of the RRT node vertices and the relationship to this with asymptotic optimality property. We also substantiate the findings with a numerical simulation presented at the end of the paper. Our work is primarily based on the analyse of scale free properties of a random network, more specifically on the power law of the asymptotic degree distribution of the network. For details of this problem and definitions, see \cite{fkp1},\cite{fkp2}.

This paper is organized as follows. In section \ref{Problem_formulation} we define the problem of motion planning and define the notion of asymptotically optimal motion planning. Following in section \ref{rrt}, we present and discuss the RRT planner in detail. Section \ref{asympop} contains the main result regarding the asymptotic optimality of RRT. Section \ref{simulation} contains a simulation results and discussion on the validity of our theoretical claim. Finally, in \ref{conclusion} we conclude and enlists the future work.

\section{Motion Planning Problem Formulation}\label{Problem_formulation}
The configuration of an autonomous agent, with $d$ degrees of freedom, can be represented as a point in a $d$-dimensional space, called the configuration space $\mathbb{C}$, which is locally like a $d-$dimensional Euclidean space $\mathbb{R}^d$. A configuration $q$ in $\mathbb{C}$ is free if the robot placed at $q$ does not collide with the obstacles in $\mathbb{C}$. Let us augment this definition of configuration space with its tangent bundle. A tangent bundle of $\mathbb{C}$ is defined as $T(\mathbb{C}) = \cup_{q \in \mathbb{C}}T_q(\mathbb{C})$, where $T_q(\mathbb{C})$ is the collection of all tangent vectors at $q$. 

The configuration space together with its tangent bundle is called a state space $X$, in which a state $x \in X$ is simply defined as $x = (q, \dot{q})$. Holonomic constraints can be defined as $h_i(q,t) = 0$. Non-holonomic constraints require the use of rate variables and or inequalities, that is $l_i(q,\dot{q},t) = 0$ or $l_i(q,\dot{q},t) < 0$. Differential constraints can be written in Lagrangian dynamics as a set of equations of the form $g_i(q,\dot{q},\ddot{q},t) = 0$, additionally involving acceleration. The use of state space  formulation allows representations of the dynamic constraints as a set of $m$ equations $G_i(x, \dot{x}) = 0$, $m < 2d$, where, $d$ is the dimension of the configuration space. These equations can be rewritten in the form $\dot{x} = f(x,u)$, where $u \in U$, where $U$ is the set of allowable control inputs to the system. The equations thus describe the state transitions resulting from a control input.

To define state space obstacles apart from defining $x \in X_{\text{obs}} \Leftrightarrow q \in C_{\text{obs}}$ for $x = (q, (\dot q))$, we need to define the reachable set from an initial configuration. For the system defined by the expression $\dot{x} = f(x,u)$, a state $x'$ can be obtained by applying a control input $u$ over time $t$ from an initial state of $x_0$. The set of all possible $x'$ is called the reachable state of $x_0$ for a time $t$. For each state $x$, among the set of reachable states, one can define future collision states $X_{\text{fc}}$ and free state space $X_{\text{free}} = X \setminus X_{\text{fc}}$. Note that, $X_{\text{fc}}$ grows as the speed of the system increases and may look rather different from $X_{\text{obs}}$. This makes finding a valid kinodynamic trajectory more difficult. 

The goal of a motion planner is to find a trajectory $x(t) \in X_{\text{free}}$ from an initial state $x_{\text{init}} \in X_{\text{free}}$ to a final state $x_{\text{final}} \in X_{\text{free}}$ and to find a time parametrized function of control input $u(t)$ that results in such a trajectory. For convergence issues, a general goal subset $x_{\text{goal}}$ is assumed, rather than a specific $x_\text{{final}}$. Clearly, $x_{\text{final}}\in{x_{\text{goal}}}$.

\section{RAPIDLY EXPLORING RANDOM TREE}\label{rrt}
Rapidly Exploring Random Tree (RRT) has been shown to be very effective in solving robot motion planning problems in a complex state space with kinodynamic motion constraints. RRT is introduced in \cite{lavalle},\cite{lavalle_2} as an efficient data structure and sampling scheme to quickly search high dimensional spaces that have algebraic constraints (arising from the obstacle) and differential constraints (arising from nonholonomy and system dynamics). The algorithm incrementally builds a tree whose nodes are different states of the robot/vehicle. These nodes are added randomly to the tree until one of the node comes close enough to any of the states in $x_{\text{goal}}$. Next, that goal state is added to the tree and a solution trajectory connecting $x_{\text{final}} \in x_{\text{goal}}$ and $x_{\text{init}}$ can be found by backtracking the nodes. The edges of the tree forms a feasible path or solution trajectory connecting a pair of initial and final states. 

The key idea behind RRT is to bias the tree growth towards unexplored regions of the state space by random sampling and extending tree nodes to those regions. The selection of tree nodes for expansion is heavily dependent on current spatial distribution of tree nodes within the state space. Implicitly, the nodes with larger Voronoi cells are more probable for expansion. This is because the probability that a node is selected for expansion is directly proportional to the volume measure of its Voronoi cell. The tree node extension logic is based on forward simulation of system dynamics upon random control input. In the following, we present the basic RRT algorithm. The RRT algorithm consists of two subroutines, Build-RRT (Algorithm \ref{Build RRT}) and Extend-RRT (Algorithm \ref{Extend RRT}).
\begin{algorithm}
\SetAlgoLined
 $T\cdot{init}(x_{init})$~$\rhd$~Initialize tree $T$\;
 \For{i=1,\ldots,K}
 {
  $x_{rand}\leftarrow$ \text{Random Configuration}\;
  \text{Extend} $(T,x_{rand})$
 }
 \Return{$T$}
 \caption{Build-RRT}\label{Build RRT}
\end{algorithm}

\begin{algorithm}
\SetAlgoLined
 $x_{near}\leftarrow$\text{Nearest Neighbor} $(x,T)$\;
 $x_{new} \leftarrow$\text{New State}$(x_{near},u)$\;
 \If{$x_{new}$ \text{is Not in Obstacle}}
  {
   $T\cdot{add\_vertex} (x_{new})$\;
   $T\cdot{add\_edge} (x_{near},x_{new},u_{new})$\;
  \If{$x_{new} \in X_{goal}$}
  {
    \Return{Reached}\;    
    \Else
    {
    \Return{Continue}
    }
    }
    }
 \caption{Extend RRT}\label{Extend RRT}
\end{algorithm}

The Build-RRT algorithm initially samples a random state or $x_{rand}$. In the Extend-RRT function, a nearest node $x_{near}$ from the tree to the generated random state $x_{rand}$ is selected for future expansion. The function New State($x_{near},u$), does a forward simulation of system dynamics for $\Delta t$ time period by applying a control input $u \in U$, where $U$ is a finite input set, to the state at $x_{near}$. This input can be chosen at random or can be selected from all possible inputs by choosing one which yields a new state $x_{new}$ which is as close as possible to $x_{rand}$. The selection of input $u$ can also be based on minimization of some performance criterion. Further, $x_{new}$ is checked for collision and system constraint violation. If $x_{new}$ does not collide with any of the state space obstacles and satisfies all the constraints, then $x_{new}$ is added to the tree as a child node of $x_{near}$; otherwise it is discarded. If $x_{new} \in x_{goal}$, the algorithm stops. This way the vertices or nodes of the RRT tree eventually forms a large connected component within $X_{free}$ and can come arbitrarily close to any specified $x_{goal}$. Note that, since the algorithm is only probabilistically complete, the algorithm continues to search for a solution until it finds a close enough node to $x_{goal}$. Heuristic termination condition can be added to stop the algorithm after a certain number of iterations.

Before we move on to the asymptotically optimality analyse of the RRT planner it is worthwhile to look into a few aspects of the algorithm. For example, although RRT vertices can come arbitrarily close to any state in $x_{goal}$, the convergence to the solution trajectory may be slow. One solution may be to sample certain $x_{goal} \in X_{goal}$ repeatedly in random sampling stage so that this introduces a biased sample generation towards $X_{goal}$. In such a way, a quick convergence towards the solution trajectory can be found, although it may lead RRT vertices to fall into a trap formed by certain spatial distribution of state space obstacle.

The choice of step size and $\Delta t$ used in forward simulation of the Extend function is also important. If the obstacles are located far away and system dynamics is considerably simple, then a larger step size can be used. In such cases, instead of attempting to extend $x_{near}$ by incremental step size, one can use the Extend function repeatedly until the extension is no longer possible. That is, a forward simulation produces states that violate system constraints or collide with the obstacle. Obviously, the success of this heuristics method is heavily dependent on the chosen problem. As reported in \cite{linderman}, this modification works best for holonomic planning.

RRT can be used as a single directional or bidirectional planner. A single directional RRT planner builds a tree starting from $x_{init}$ towards $X_{goal}$. A bidirectional planner however builds two separate trees. One from $x_{init}$ to $X_{goal}$, and the other is from any state $x \in X_{goal}$ to $x_{init}$. Once the distance between a pair of vertices from either tree is within a certain tolerance bound, the two trees gets connected and forms a large single connected component.  For holonomic planning, a bidirectional planner is faster than the single directional planner. This heuristics, however, does not work for non-holonomic or for differential systems, because of reachability problem.

Finally, an important subroutine in the RRT algorithm is the nearest neighbour search. Currently, the basic RRT's nearest neighbour search algorithm works based on the principle of exhaustive search. For a fixed dimension $d$, if at any instant the RRT tree contains $n$ vertices, the task of searching for a true nearest neighbour node from a given query point has $O(n)$ time complexity. This approach, although linear with the number of vertices, eventually slows down (in the sense of computational time) the growth of RRT as the number of vertices increases. It should be noted that any approximate nearest neighbour searching degrades the RRT's exploration capability. For example, suppose at every iteration, a random vertex is selected as the nearest neighbour and RRT is expanded from this vertex. The probability of selecting any node from the $n$ number of nodes is $1/n$. While for a basic RRT, probability of selecting any node $q$ as nearest neighbour for some query point, is $O(\text{Vor}(q))$, where $\text{Vor}(q)$ is the associated Voronoi cell of the node $q$. Since the node with largest Voronoi area has a higher probability of selection, the RRT vertices tend to grow to the region of largest Voronoi regions or, in other words, to the unexplored regions of state space. Clearly, selecting a random vertex as nearest neighbour is a compromise with the search of unexplored areas. However, after a large number of iterations if $O(\text{Vor}(q))\approx 1/n$ then selection of any vertex, as nearest neighbour, may be effective, specially for holonomic cases.

\section{Asymptotic Optimality of RRT}\label{asympop}
Sometimes it is necessary to find an optimal solution for certain motion planning problems. For example, a motion planning task may contain additional objectives such as minimization of control effort or determination of a shortest path to minimize fuel cost. Except for a very simplistic situation, in general, the search for an optimal solution is difficult. For example, consider the general problem of shortest path calculation in $3$-dimension among polyhedral obstacles. There are a number of solutions that exist only for a restricted class of problems, as can be found in \cite{mount}. These solutions are algorithms with polynomial time complexity, applicable in an environment which consists of a few convex polyhedra. In \cite{papa} a polynomial time approximate algorithm is given which finds a sub-optimal path in $3$-dimension. In \cite{sharir_2} a $2^{2^{O(n)}}$ complexity algorithm is proposed to construct the shortest path by reducing the problem to an algebraic decision problem. The best bound is obtained in \cite{reif}, which is $2^{n^{O(1)}}$. 

These results indicate that finding an optimal path in a very general condition is computationally expensive. Since sampling based algorithms are used to avoid computational complexity issues that are present in obstacle avoidance and motion planning problems, it is natural to investigate the quality of the path returned by a sampling based motion planner. In this context, only \cite{karaman} discusses and show the asymptotic optimality properties of various sampling based motion planning algorithms including RRT. 

For sampling based motion planning algorithms, optimality of the solution can only be analysed in an asymptotic sense. Naturally, a definition of the asymptotic optimality of the solution is required. We outline the definition according to \cite{karaman}.
\begin{defn} An algorithm is asymptotically optimal if, for any path planning problem and cost function $c$ that admits a robustly optimal solution with finite cost $c^*$, $P(\{\limsup\limits_{n\rightarrow \infty}Y_n = c^*\})=1$, where $Y_n$ is the algorithmic solution after $n$ iterations.
\end{defn}

In the argument presented in \cite{karaman}, first it is defined that the set of vertices that contains the $k^{\text{th}}$ child of the root along with all its descendants in the tree is called the $k^{\text{th}}$ branch of the tree.  Then, it is shown that a necessary condition for the asymptotic optimality of RRT is that infinitely many branches of the tree contain vertices outside a small ball centred at the initial condition. This is captured in the following lemma.
\begin{lemma}
Let $0<R<\inf_{y \in \mathbb{X}_{goal}}||y-x_{init}||$. The event $\{\lim_{N\rightarrow \infty}Y_n = c^*\}$ occurs only if the $k^{\text{th}}$ branch of the RRT has vertices outside the $R$-ball centred at $x_{\text{init}}$, for infinitely many $k$.
\end{lemma}

We observe that a necessary condition for optimality should be applicable to all the nodes in the tree and not just only $x_{\text{init}}$. Consequently, we hypothesize that for a tree like planner to return an asymptotically optimal solution it is a necessary condition to have a mechanism which creates infinite degree of the tree vertices.

Next, it is shown that the RRT algorithm does not satisfy the condition for asymptotic optimality. In the detailed analysis, as can be found in \cite{karaman}, the authors note that if $U = \{X_1,X_2,\ldots, X_n\}$ be a set of independently sampled and uniformly distributed points in the $d$-dimensional unit cube $[0,1]^d$, and if $X_{n+1}$ is a newly sampled i.i.d. point, then the probability, that among all points in $U$ the point in $X_i$ is the one that is closest to $X_{n+1}$, is $1/n$. An immediate consequence of this result is that each vertex of the RRT has unbounded degree almost surely as the number of samples approaches infinity. 

We note that, although in the sample configuration phase of RRT algorithm, a new sample is chosen according to a uniform distribution, the selection of a node as nearest neighbour to this sample also depends on the existing geometric Voronoi partition of all nodes. The node with higher Voronoi volume will have a higher chance of being selected as the nearest neighbour. Only when the Voronoi partition is approximately equal for every node, the selection of the nearest neighbour will be independent of spatial distribution of nodes. Secondly, even if the nearest neighbour selection is according to a uniform distribution, this does not guarantee that each vertex will have an unbounded degree. We present a theoretical argument as well as a simulation to support our claim. Additionally, an extension of this observation leads to an alternate proof of non-optimality of RRT algorithm in an asymptotic sense.

\begin{theorem}\label{theorem_1}
The solution trajectory returned by the RRT algorithm is not asymptotically optimal.
\end{theorem}

\begin{proof}
Before giving the proof, in the following, some definitions are presented which was used in the deduction.

\begin{defn}
Let the state space $X$ be a metric space. We assume $X$ is path connected. That is, any pair of states $A,B \in X$ can be connected by a continuous path $\zeta \in X$ which is appropriately parametrized. Denote $d(x,y)$ as a distance metric in $\mathbb{X}$. For simplicity, we consider a holonomic planning problem.
\end{defn}
\begin{defn}
Let the tree nodes, defined as $\{x_0,x_1,\ldots,x_i,\ldots\}$, form an incremental construction of the RRT tree where indexes are identified with time instances.
\end{defn}
\begin{defn}
Given a region $\mathcal{R}$ and a set of points in $P_X \in \mathcal{R}$, there exists a unique arrangement of partition of the region, denoted as Voronoi partitions $V_{P_X}$ for each $x \in P_X$ as the set of points closer to $x$ than any other point $y \in P_X$.
\end{defn}

Note that, for an $n^{\text{th}}$ incremental construction, the obtained RRT vertices are defined as $\{x_0,x_1,\ldots,x_n\}$. The Voronoi volume of $x_i$, $V_{x_i,n}$ is a non increasing function in $n$. That is, $V_{x_i,n} \geq V_{x_i,n+1}$. We emphasize this fact continuing in the similar way as in \cite{fkp2}.

Let $x_n$ be a point, uniformly at random, in a convex region $R$. Let $x_c$ be a point in $R$ considered as the origin. Now, we can divide $R$ into $f(d)$ number of disjoint cones, defined as $c_1,\ldots,c_{f(d)}$ \cite{yukich}.
Suppose $x_n$ falls into some such cone $c_i$. If $y$ is any other point in $c_i$ then let $R'$ be the set of points in $R$ which is closer to $x_n$ than $x_c$.

Now we can write that $R\setminus R'$ contains a segment of $c_i$ which is shrunk by a factor of two in every dimension. So, in this case,
\begin{equation}
\mu(R\setminus R') \geq \mu(c_i)/2^d
\end{equation}
where, $d$ is the dimension and $\mu$ is the $d$ dimensional Lebesgue measure. Note that, the probability that $x_n$ will be within any of such $c_i$s is equal to the volume ratio of $c_i$ to $R$. Taking the expectation over all $c_i$,
\begin{equation}
E(\mu(R\setminus R') = \sum_{i= 1}^{f(d)}p_i \mu(c_i)/2^d
\end{equation}
where, $p_i$ is the probability that $x_n$ will fall within of the $c_i$. Since, $p_i = \mu(c_i)/\mu(R)$, therefore,
\begin{eqnarray}
E(\mu(R\setminus R')) &\geq & \sum_{i= 1}^{f(d)}{\mu(c_i)}^2/(2^d \mu(R)) \\
& = & \mu(R)/(2^d f(d))
\end{eqnarray}
Naturally, this implies $E(\mu(R')) \leq (2^d f(d) - 1)(\mu(R))/(2^d f(d))$. 

Let, at any instant $t_0$, the Voronoi volume of the $i^{\text{th}}$ node is $\mu(V_{x_i,n})^0$ and $\gamma = (2^d f(d) - 1)/(2^d f(d))$. From recursion, 
\begin{equation}\label{expt1}
E(\mu(V_{x_i,n}^{t_0+k})) \leq \gamma^k \mu(V_{x_i,n}^0)
\end{equation}

Now, we try to show an upper bound on the probability of the event that a RRT vertex increases its degree from $k$ to $k+1$ at some instant.

Let the degree of $x_i$ be $k$. Then, in order to have a degree of $k+1$ in the next instant of the tree generation, either of the following event should happen. For any volume $v_x$, either the Voronoi volume of $V_{x_i,n}$ is more than $v_x$, or the dispersions created by the other vertices of the tree should have a volume less than $v_x$. Taking an union bound,
\begin{equation}
p(\text{deg}(x_i) \geq k+1) \leq p({\mu(V_{x_i,n})}^k \geq v_x) + p(\mu(V_{x_j,n}) \leq v_x) 
\end{equation}
$\forall j \in \{x_0,x_1,\ldots,x_n\}, j\neq i$
Now, using Markov's inequality, we have $ p({\mu(V_{x_i,n})}^k \geq v_x) \leq E(\mu(V_{x_i,n})^k)/v_x $. 

For the other probability we can use again the union bound over all $x_0,\ldots,x_n$ and, $p(\mu(V_{x_j,n}) \leq v_x) \leq \alpha n v_x$, for some constant $\alpha$.
Therefore,
\begin{equation}
p(\text{deg}(x_i) \geq k+1) \leq E(\mu(V_{x_i,n})^k)/v_x +  \alpha n v_x
\end{equation}
If, $\mu(V_{x_i,n})^0 = \mathcal{O}(1/n)$, using (\ref{expt1}), $v_x$ can be chosen optimally as $v_x = \gamma^{k/2}\mathcal{O}(1/ n)$. The choice of $\mu(V_{x_i,n})^0 = \mathcal{O}(1/n)$ is justified as it is shown in \cite{janson}.
Then summing over all the $n$ nodes, the expected number of nodes with degree at least $k+1$ is at most $\alpha \gamma^{k/2}\mathcal{O}(1)$.
Hence, it is evident that the degree distribution of the vertices of the RRT tree follows a power law. This implies that the RRT tree will contain fewer and fewer vertices with higher and higher degree. In other words, the chance that nodes will create an infinite number of branches becomes lesser and lesser as the number of iterations increases.

However, since for a sampling based planner, the measure of the event that the planner finds the optimal path within any finite iteration is zero; hence, there must be a sequence of paths which converges to the optimum path. For existence of such a sequence, it is necessary that the vertex degrees should be infinite, which is in opposition to the present findings. Hence, an RRT planner with such node degree distribution cannot be an asymptotically optimal planner. 
\end{proof}

As a corollary to the above theorem, it can be said that a sampling based planner for which its degree distribution follows a power law cannot be a candidate for optimal motion planning.

In the following section, to further establish Theorem \ref{theorem_1}, a numerical simulation is provided which clearly shows that the cumulative degree distribution follows a power law, if not exponential.

\section{Results and Discussion}\label{simulation}

To validate the theoretical claim the following simulation of an RRT is performed. The system is a simple non-holonomic car, shown in Figure \ref{vehicle}, and defined by the following equations:
\begin{eqnarray}\label{Chapter4_Car_Equaitons}
\dot x &=& v~\cos \theta \\
\dot y &=& v~\sin \theta\\
\dot {\theta} &=& (v/L)~\tan\phi
\end{eqnarray}

\begin{figure}
 \centering
 \def\svgwidth{200pt}
 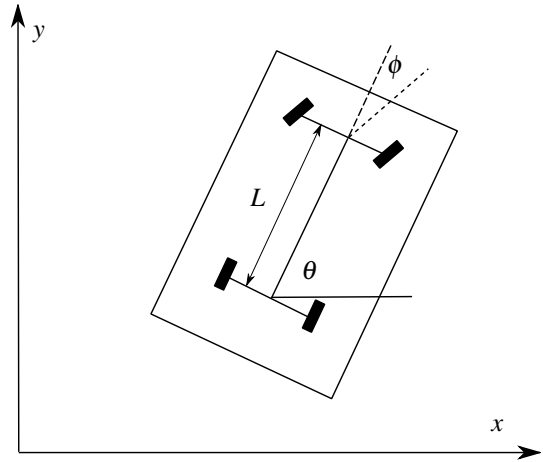
 \caption{A Non-holonomic car}
 \label{vehicle}
\end{figure}
where, $L$ is the length between the wheels, $v$ is the velocity and $\phi$ is the steering angle. The position of the car and its orientation is denoted by a $3$-tuple $(x,y,\theta)$. The control variables are velocity $v$ and steering angle. The configuration space is $\mathbb{R}^2\times S^1$. Figure \ref{Chapter3_Pic_2} shows the histogram of vertex out degree of the RRT planner after $5000$, $10000$, $15000$, and $20000$ iterations. The vertex out degree is an indication of how many branches are emitted from a single vertex. Clearly, the histogram shows that a vertex with very high value of out degree is less frequent compared to the vertices with lesser out degree. 

\begin{figure}
  \centering
  \subfigure[]{\includegraphics[height=4cm,width=4cm]{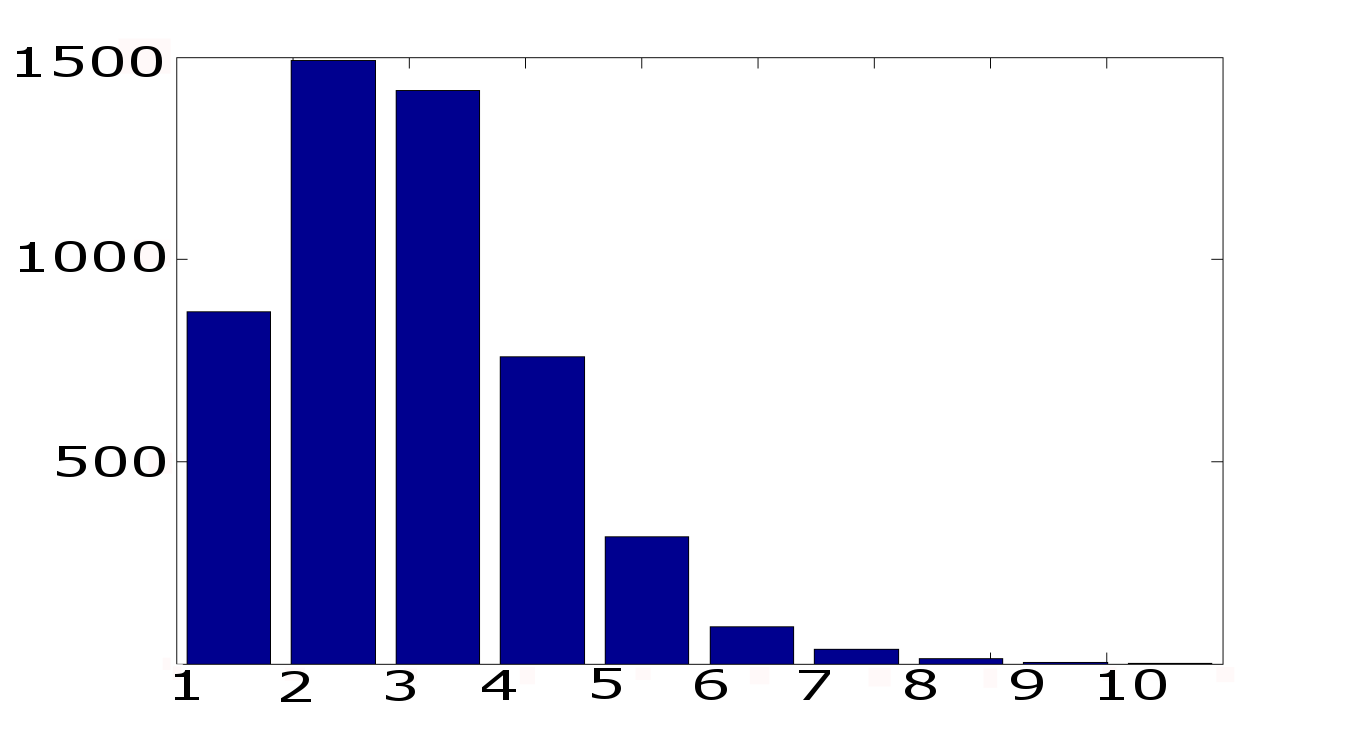}}\hspace*{0.1cm}
  \subfigure[]{\includegraphics[height=4cm,width=4cm]{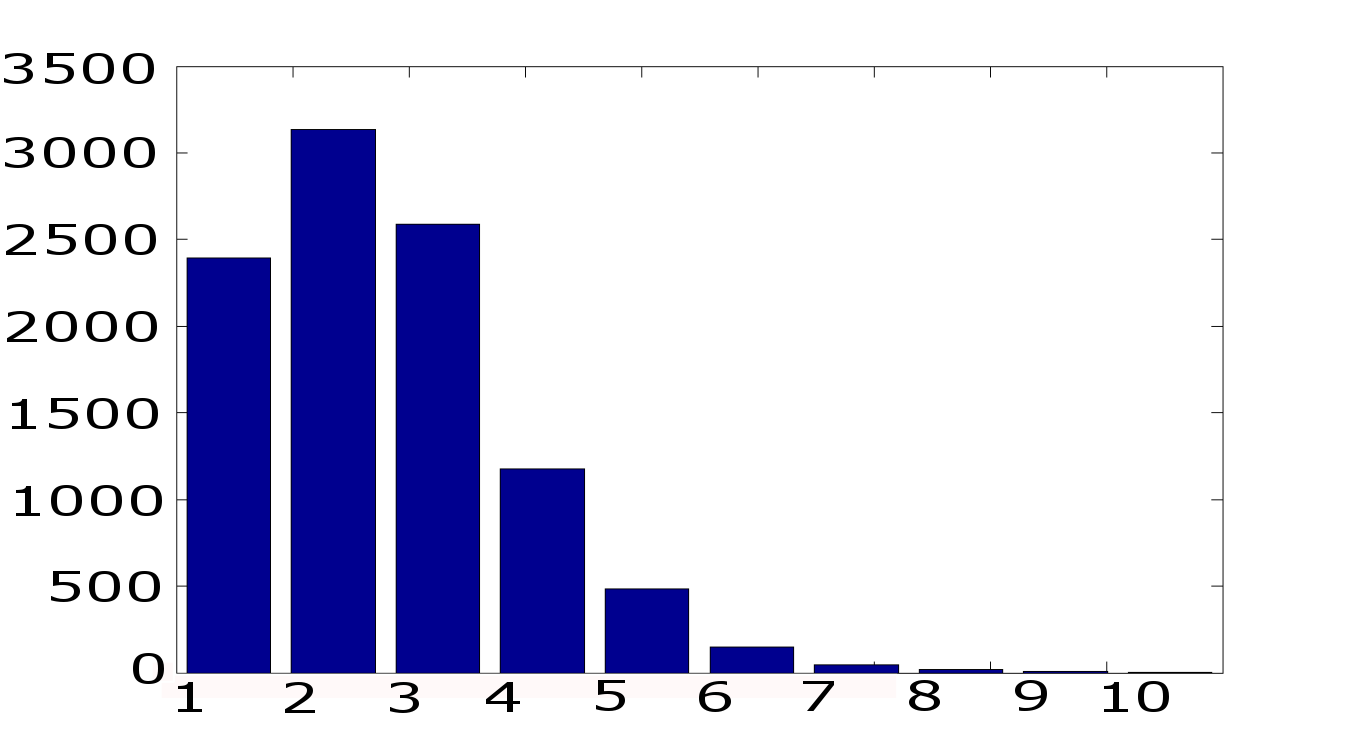}}\\
  \subfigure[]{\includegraphics[height=4cm,width=4cm]{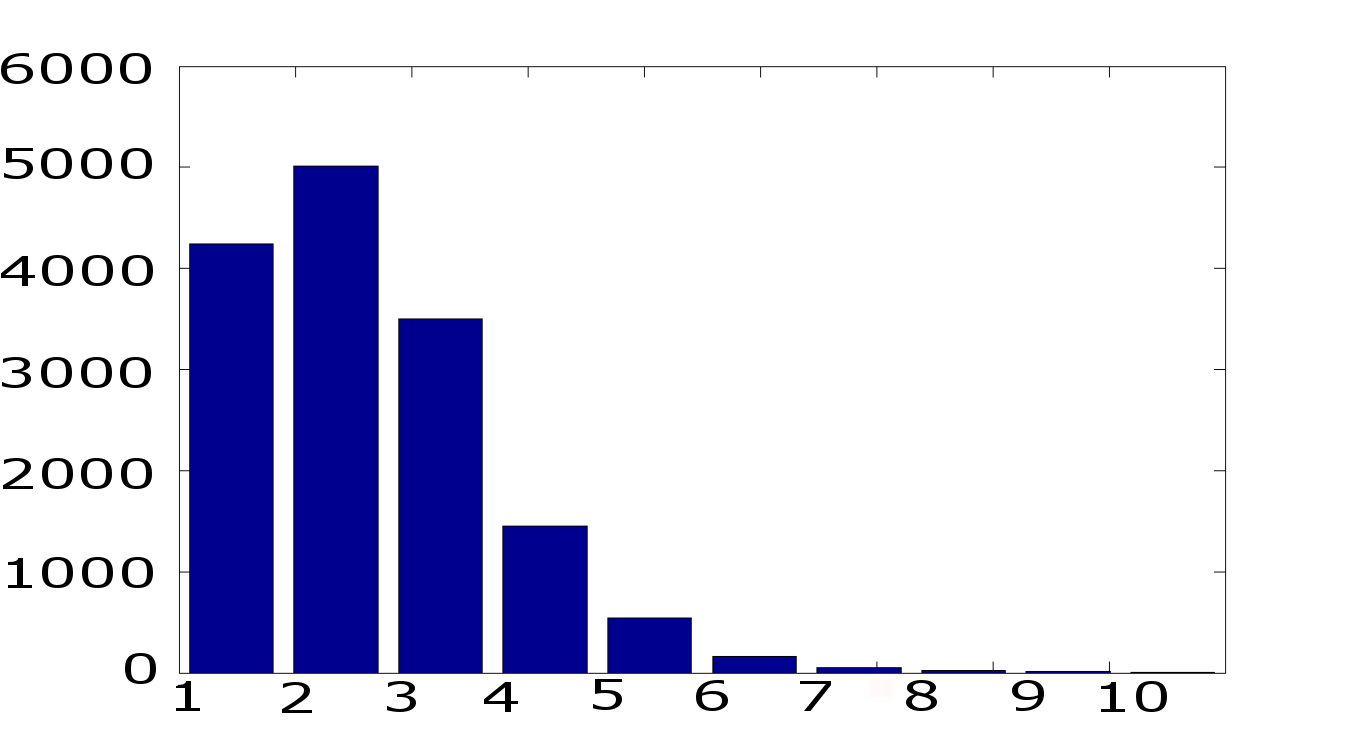}}\hspace*{0.1cm}
  \subfigure[]{\includegraphics[height=4cm,width=4cm]{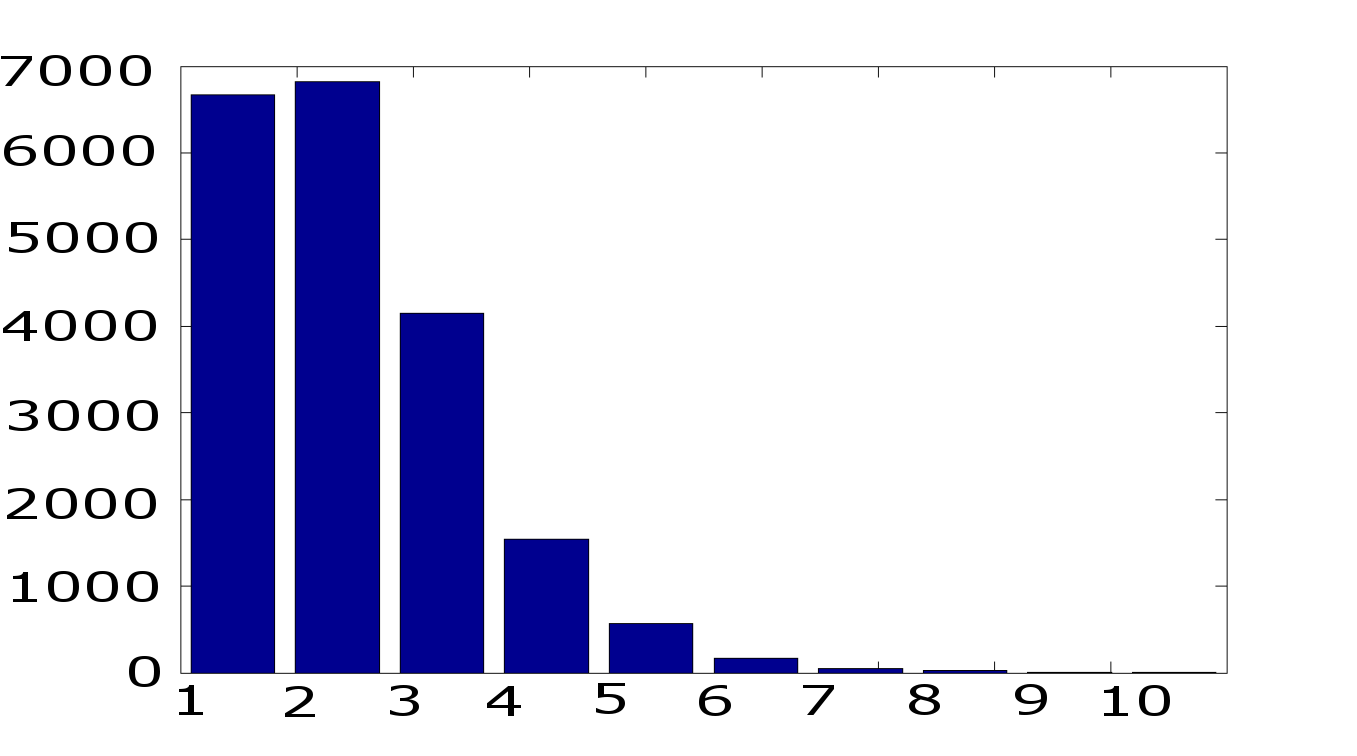}}\\
  \caption{(a) Histogram of RRT vertices out degree after $5000$ iterations (b) after 10000 iterations (c) after 150000 iterations (d) after 20000 iterations. The abscissa represents node out degree and the ordinate represents the number of RRT nodes having the same out degree.}
  \label{Chapter3_Pic_2}
\end{figure}
The simulation strongly supports the theoretical claim. Note that this characteristics of RRT vertices are actually independent of system dynamics. This is because the additional node generation is only dependent on random samples and nearest neighbour subroutines.
\section{Conclusion}\label{conclusion}
In this paper we investigated the asymptotic optimality of an existing randomized sampling based algorithms, namely RRT. We showed that the degree distribution in RRT planner follows a power law. This also implies that the RRT planner is not an asymptotically optimal planner which is consistent with existing results. We also present a simulation which strongly supports the theoretical claim. It would be interesting to analyse other sampling based algorithms in the perspective of asymptotic degree distribution.

\end{document}